\documentclass[runningheads]{llncs}
\usepackage{graphicx}
\usepackage{subfigure}

\usepackage{amsmath,amssymb,amsfonts}

\newcommand{\argmax}{\operatornamewithlimits{argmax}}

\newcommand{\E}{\mathbb{E}}
\newcommand{\I}{\mathbb{I}}

\newcommand{\Z}{\mathbb{Z}}

\newcommand{\cA}{\mathcal{A}}

\newcommand{\cM}{\mathcal{M}}

\newcommand{\prob}{\text{Pr}}

\newcommand{\OPT}{{\it OPT}}

\usepackage{algorithm}
\usepackage{algorithmicx}
\usepackage{algpseudocode}
\algrenewcommand\algorithmicrequire{\textbf{Input:}}
\algrenewcommand\algorithmicensure{\textbf{Output:}}

\usepackage{ifthen}
\usepackage{xcolor}

\usepackage{thm-restate}

\newcommand{\compilehidecomments}{false}

\ifthenelse{\equal{\compilehidecomments}{true}}{%
    \newcommand{\wei}[1]{}
    \newcommand{\haoyu}[1]{}
}{
    \newcommand{\wei}[1]{{\color{blue!50!black}  [\text{Wei:} #1]}}
    \newcommand{\haoyu}[1]{{\color{brown!60!black} [\text{Haoyu:} #1]}}
}

\newcommand{\compilefullversion}{true}
\ifthenelse{\equal{\compilefullversion}{false}}{%
    \newcommand{\OnlyInFull}[1]{}
    \newcommand{\OnlyInShort}[1]{#1}
}{%
    \newcommand{\OnlyInFull}[1]{#1}%
    \newcommand{\OnlyInShort}[1]{}%
}%

\begin{document}
\title{Stochastic One-Sided Full-Information Bandit}
%
%
\author{Haoyu Zhao\inst{1} \and
Wei Chen\inst{2} }
%
%
\institute{Institute for Interdisciplinary Information Sciences, Tsinghua University, Beijing, China,  \email{zhaohy16@mails.tsinghua.edu.cn} \and
Microsoft Research, Beijing, China, \email{weic@microsoft.com}}
\maketitle              
\begin{abstract}
In this paper, we study the stochastic version of the one-sided full information bandit problem, 
    where we have $K$ arms $[K] = \{1, 2, \ldots, K\}$, and playing arm $i$ would gain reward from an unknown distribution for arm $i$
    while obtaining reward feedback for all arms $j \ge i$.
    One-sided full information bandit can model the online repeated second-price auctions, where the auctioneer could select the reserved price in each round and the bidders only reveal their bids when their bids are higher than the reserved price. 
    In this paper, we present an elimination-based algorithm to solve the problem. Our elimination based algorithm achieves distribution independent regret upper bound $O(\sqrt{T\cdot\log (TK)})$, and distribution dependent bound $O((\log T + \log K)f(\Delta))$, where $T$ is the time horizon, $\Delta$ is a vector of gaps between the mean reward of arms and the mean reward of the best arm, and $f(\Delta)$ is a formula depending on the gap vector that we will specify in detail. 
    Our algorithm has the best theoretical regret upper bound so far.
    We also validate our algorithm empirically against other possible alternatives.

\end{abstract}
\section{Introduction}

Stochastic multi-armed bandit (MAB) has been extensively studied in machine learning and sequential decision making. 
The most simple version of this problem consists of $K$ arms, where each arm has an unknown distribution of the reward. The task is to sequentially select one arm at each round so that the total expected reward is as high as possible. In each round, we will gain the reward and only observe the reward of the arm we choose. The trade-off between exploration and exploitation appears extensively in the MAB problem: On the one hand, one might try to play an arm which is played less to explore whether it is good, and on the other hand, one might choose to play the arm with the largest average reward so far to cumulate reward. 
MAB algorithms are measured by their {\em regret}, 
which is the difference in expected cumulative reward  between the algorithm
    and the optimal algorithm that always chooses the best arm.

A variant of the stochastic MAB problem is the {\em one-sided full-information bandit}, where there is a set of arms $1,2,\dots,K$ and at round $t$ we choose arm $I_t$, we will gain the reward of $I_t$ at time $t$ and observe the rewards of all arms $i \ge I_t$ at time $t$
    (Section \ref{sec-model}). 
The adversarial version of the one-sided full-information bandit is first introduced in \cite{DBLP:conf/colt/Cesa-BianchiGGG17}, and in this paper, we study it stochastic version. 

\setcounter{footnote}{0}

One-sided full-information bandit can find applications such as in online auction.
Consider for example the second-price auction with a reserve price.
In each round, the seller (or auctioneer) sets a reserve price from a finite set of reserve price choices. Each buyer (or bidder) draws a value from its valuation distribution 
    (unknown to the seller),
    and only submits her value as the bid when her value is at least as high as 
    the reserve price. The seller would observe these bids, give the item to
    the highest bidder and collect the second highest bid price (including the reserve price)
    as its reward from the highest bidder.
In this case, we can treat each reserve price as an arm. In each round $t$ after
    the seller announces the reserve price $r_t$, she will see all bids higher than $r_t$, 
    and thus she would know the reward she could collect for all reserve prices
    higher than or equal to $r_t$, which corresponds to the case of one-sided
    full-information feedback.\footnote{Note that the second-price auction is truthful
    in a single round, but in multi-rounds, it may not be truthful since the bidders
    may want to lower their bids first so that the seller would learn a lower reserve
    price. The truthfulness is not the main concern of this paper and its discussion
    is beyond the scope of this paper.}

In this paper, we present an elimination-based algorithm for the stochastic one-sided 
    full-information bandit and 
    prove the distribution-independent bound as $O(\sqrt{T(\log T + \log K)})$ and
    the distribution-dependent bound as 
    $O((\log T + \log K)f(\Delta))$, where $T$ is the time horizon, $\Delta$
    is a vector of gaps between the mean reward of arms and the mean reward of the best arm, and $f(\Delta)$ is a formula depending on the gap vector
    that we will specify in Theorem~\ref{dep} (Section \ref{sec-alg-ana}).
We also adopt an existing analysis to show a distribution-independent regret lower bound of 
    $\Omega(\sqrt{T\log K})$ for this case
    (Section~\ref{sec-low}), which indicates
    that our algorithm achieves almost matching upper bound.
We conduct numerical experiments to show that our algorithm significantly outperforms
    an existing algorithm designed for the adversarial case
    (Section \ref{sec-exp}).
The empirical results also indicate that a UCB variant has better empirical performance, but it so far has no tight theoretical analysis, and thus our elimination-based algorithm is still the one with the best theoretical guarantee.

Due to space constraint, some proofs are moved to a supplementary material submitted together with the main paper.

%
\subsection{Related Work}
\textbf{Multi-armed bandit:}     
Multi-armed bandit (MAB) is originally introduced by Robbins~\cite{Robbins52}, and 
    has been extensively studied in the literature (c.f. \cite{BF85,BCB12}).
MAB could be either stochastic, where the rewards of arms are drawn from unknown distributions,
    or adversarial, where the rewards of arms are determined by an adversary.
Our study in this paper belongs to the stochastic MAB category.
The classical MAB algorithm includes UCB \cite{AuerCF02} and 
    Thompson sampling \cite{thompson1933likelihood} for the stochastic setting
    and EXP3 \cite{AuerCFS02} for the adversarial setting.

\textbf{Multi-armed bandit with graph feedback structure:} 
One-sided full-information bandit can be viewed as a special case of the MAB problem with graph feedback structure.
The arm feedback structure can be represented as a graph (undirected or directed, 
    with or without self-loops), where vertices are arms, and when an arm is played,
    the rewards of all its neighbors (or out-neighbors) can be observed.
The one-sided full-information bandit corresponds to a feedback graph with directed
    edges pointing from arm $i$ to arm $j$ for all $i\le j$.
The first paper that introduces MAB with graph feedback is \cite{mannor2011bandits}. The authors of this paper use the independent number and the clique-partition number to derive the upper and lower bound for the regret. The main results of \cite{mannor2011bandits} is the upper and lower bound for the regret for undirected graph feedback MAB problem. 
Although the bound is tight in the undirected case,  there is a gap between the regret upper and lower bounds for directed graphs. 
When translated to our one-sided full information setting, their regret upper bound is $\tilde O(\sqrt{KT})$ but the lower bound is $\tilde \Omega(\sqrt{T})$, 
 which are not as tight as we provide in this paper in both upper and lower bounds.
In \cite{alon2013bandits}, the authors consider the adversarial MAB with general directed
    feedback graph and close the gap up to some logarithmic factors.
However, when applying their results to the one-sided full-information bandit setting,
    their upper and lower bounds are all worse than ours by a logarithmic factor.
Moreover, we provide distribution-dependent bound that only works for the stochastic setting.
One-sided full-information bandit is originally proposed in~\cite{DBLP:conf/colt/Cesa-BianchiGGG17}, which studies the adversarial setting
and proposes a variant of EXP3 algorithm EXP3-RTB to solve this problem in the
adversarial setting. Their work focuses on the more general bandit on metric space, and
    ignores the difference in the logarithmic factors.
Stochastic MAB with undirected graph feedback is studied in \cite{CaronKLB12}, 
    which proposes a variant of UCB algorithm UCB-N that essentially acts as UCB but updates
    all observed arms instead of only the played arm in each round.
The authors derive a regret upper bound based on the cliques in the feedback graph, but in the one-sided full-information setting the cliques are 
    reduced to singletons and their regret result is reduced to the classical UCB, which is significantly worse than the regret of our algorithm.
We include UCB-N in our experiments, which demonstrate good performance of UCB-N, but we cannot provide a better theoretical regret bound for it, and this task is left as a future work item.
    
    \section{Model}\label{sec-model}
    In this section, we specify a multi-armed bandit model called `one-sided full information bandit', which is highly related with the online auction problem. Suppose that there are $K$ arms $[K] = \{1,2,\dots,K\}$ in total. Each time we play the arm $I_t$ at round $t$, we will observe the value of arm $i$, denoted as $X^{(t)}_i$, for all $i\ge I_t$. We study this problem under the stochastic settings, i.e. in each round $t$, the realized value $X^{(t)}_i$ is drawn from a distribution $\nu_i$, and $X^{(t)}_i$ is independent to $X^{(t')}_i$, for all $t' < t$. The formal definition of the bandit model is given as follow.
    \begin{definition}[One-sided Full Information Bandit]
        There is a set of arms $\{1,2,\dots,K\}$, and for each arm $i\in [K]$, it corresponds to an unknown distribution $\nu_i$ with support $[0,1]$, where $\nu_i$ is the marginal distribution of $\nu$ with support $[0,1]^K$. In each round $t$, the environment draws a reward vector $X^{(t)} = (X^{(t)}_1,\dots,X^{(t)}_K)$, where $X^{(t)}$ is drawn from distribution $\nu$. The player then chooses an arm $I_t$ to play, gains the reward $X^{(t)}_{I_t}$ and observes the reward of arms $I_t,I_t+1,\dots,K$, i.e. observes $X^{(t)}_{i},\forall i \ge I_t$.
    \end{definition}
    \begin{remark}
        In the definition, we explicitly give the joint distribution $\nu$ to describe the value distribution of arms, and denote arm $i$'s reward distribution 
        $\nu_i$ as the marginal distribution of $\nu$.
        This is to emphasize the fact that the distributions corresponding to different arms can be correlated.
        
    \end{remark}
    The performance of the multi-armed bandit algorithm is measured by regret. In the stochastic bandit scenario, people will use the pseudo-regret to measure the performance more often. The pseudo regret is defined as follow,
    \begin{definition}[Pseudo-regret]
        Let $I_t$ denote the arm that is chosen by algorithm $\cA$ to play at round $t$, then the pseudo-regret of the algorithm $\cA$ for $T$ rounds is defined as $\E[\sum_{i=t}^T (X^{(t)}_{i^*} - X^{(t)}_{I_t})]$, where $i^*$ denotes the best arm in expectation, i.e. $\E[X^{(t)}_{i^*}] \ge \E[X^{(t)}_i]$ for all $i\in [K]$.
    \end{definition}

    In this paper, we only consider pseudo-regret, and henceforth, for convenience, we simply use the term regret to refer to pseudo-regret in
        the remaining text.
    For convenience, we will use $\mu_i = \E[X^{(t)}_i]$ to denote the mean of the reward of arm $i$, and $\mu_{i^*}$ to denote the mean of the best arm. We will also use $\Delta_i = \mu_{i^*} - \mu_i$ to denote the difference of the mean between arm $i$ and the best arm $i^*$.

    \section{Algorithm and Regret Analysis}\label{sec-alg-ana}
    
    \subsection{Elimination Based Algorithm}
    
    In this section, we present an elimination-based algorithm to tackle the one-sided full information stochastic bandit problem. We first show an algorithm with known time horizon $T$. Our algorithm can be generally described as: We maintain a set of arms $S_t$ during the execution of the algorithm. At each round, we will play the arm that has the smallest index in $S_t$, i.e. $I_t\leftarrow \min_{i\in S_t} i$. At first, $S_1 = [K]$ is the set of all arms, and we will play arm $1$ in the first round. At each time $t$ we observe the rewards for the arms $I_t,I_{t+1},\dots,K$, update the empirical mean of each arm and update the set $S_t$ into $S_{t+1}$. 
    At each round  $t$, we will delete the arms in $S_t$ whose empirical means are much smaller than the best empirical mean in $S_t$. More specifically, we have 
    \vspace{-2mm}
    \[m_{t} = \argmax_{i\in S_{t}}\hat \mu_{i,t},\]
    \vspace{-2mm}
    where $\hat\mu_{i,t} = \frac{1}{t}\sum_{s=1}^t X_i^{(s)}$ is the empirical mean of arm $i$ after $t$ rounds, and 
    \vspace{-2mm}
    \[S_t = \{i\in S_{t-1} \mid \hat \mu_{m_{t-1},t-1} - \hat \mu_{i,t-1} \le 2\rho_t\},\]
    \vspace{-2mm}
    where $\rho_t$ is the {\em confidence radius} and $\rho_t = \sqrt{\frac{\ln (KT^2)}{2(t-1)}}$(The confidence radius $\rho_1$ at around $t=1$ is $\infty$). Our whole algorithm is shown in Algorithm \ref{alg-elim}.
    
    \begin{algorithm}[t]
        \caption{ELIM: Elimination-based algorithm with known time horizon $T$}
        \label{alg-elim}
        \begin{algorithmic}[1]
            \Require Time horizon $T$.
            \State $S_0 \leftarrow \{1,2,\dots,K\}$.
            \State $\forall i, \hat\mu_{i,0} = 0$.
            \For{$t=1,2,\dots,T$}
            \State $\rho_t \leftarrow \sqrt{\frac{\ln (KT^2)}{2(t-1)}}$.(The confidence radius $\rho_1$ at time $t=1$ is $\infty$). 
            \State $m_{t-1} \leftarrow \argmax_{i\in S_{t-1}}\hat \mu_{i,t-1}$.
            \State $S_t \leftarrow \{i\in S_{t-1} \mid \hat \mu_{m_{t-1},t-1} - \hat \mu_{i,t-1} \le 2\rho_t\}$.
            \State Play the arm $j$, where $j\leftarrow \min_{i\in S_t} i$.
            \State Observe the reward $X_i^{(t)},\forall i\ge j$.
            \State $\forall i \in S_t, \hat\mu_{i,t} \leftarrow \hat\mu_{i,t-1}\cdot\frac{t-1}{t} + X_i^{(t)}\cdot \frac{1}{t}$.
            \EndFor
        \end{algorithmic}
    \end{algorithm}


    We will show that our algorithm has distribution-independent regret bounded $O(\sqrt{T(\ln K + \ln T)})$, where the best regret bound for one-sided full information bandit till now is $O(\sqrt{T \ln K \ln T})$, which is implied in \cite{DBLP:conf/colt/Cesa-BianchiGGG17}. Besides the distribution-independent bound, we also give a distribution-dependent bound. The following two theorems show our results, and their proofs will be provided in the next section.

    \begin{restatable}{theorem}{distindepthm}{\rm\bf (Distribution independent regret bound)}\label{indep}
        Given the time horizon $T$, the regret of Algorithm \ref{alg-elim} is bounded by $4\sqrt{2T\ln(KT^2)}+3$.
    \end{restatable}

    \begin{theorem}[Distribution dependent regret bound]\label{dep}
        Let $\{\Delta_{(i)}\}$ be a permutation of $\{\Delta_i\mid i\le i^*\}$, such that $\Delta_{(1)} \ge \Delta_{(2)} \ge \dots \ge \Delta_{(i*)} = 0$, 
        and $C = 8\ln(KT^2)$. Given time horizon $T$, the regret of Algorithm \ref{alg-elim} is bounded by
        \vspace{-2mm}
        \begin{equation}\label{equ-dep}
            \Delta_{(1)} + \frac{C}{\Delta_{(1)}} + C\sum_{i=2}^{i^*-1}\left(\frac{1}{\Delta^2_{(i)}} - \frac{1}{\Delta^2_{(i-1)}}\right)\Delta_{(i)}+2.\vspace{-2mm}
        \end{equation}
        \vspace{-2mm}
    \end{theorem}

    Note that the standard UCB algorithm will lead to $O(\sum_{i\in [K], \Delta_i > 0} \frac{1}{\Delta_i}\ln T) = O(\sum_{i=1}^{K-1} \frac{1}{\Delta_{(i)}}\ln T)$ distribution dependent regret. 
    In Eq.~\eqref{equ-dep}, if we ignore the term $- \frac{1}{\Delta^2_{(i-1)}}$ in the summation, we could obtain the same order regret upper bound.
    Thus, the regret obtained above is typically better than the UCB regret.
    To see more clearly the difference, 
    consider the case when the best arm $i^*$ has mean $\mu_{i^*} = \frac{1}{2} + \varepsilon$ and all other arms $i\neq i^*$ have mean $\mu_i = \frac{1}{2}$, the original UCB will lead to $\frac{K-1}{\varepsilon}\ln T$ regret bound, and our algorithm will lead to $\frac{8\ln(KT^2)}{\varepsilon} + 2 + \varepsilon$ regret bound.
    Also notice that in the distribution dependent bound, we only add up to $i^*$, which means that the arms which have indices larger than $i^*$ will not contribute explicitly to the regret upper bound. This directly shows that the location of the best arm matters in our algorithm for one-sided MAB model.
    
    \begin{remark}
        Although the arms with indices larger than that of the best arm do not contribute explicitly to the regret bound, they do contribute to the constant $2$ in Eq.\eqref{equ-dep} of Theorem \ref{dep}. The contribution comes from a low probability case, which is shown in the proof in the next section.
    \end{remark}

    In Algorithm \ref{alg-elim}, we assume that we know the time horizon $T$. Now, we apply the standard `doubling trick' to get an algorithm with unknown time horizon $T$, which is shown in Algorithm \ref{alg-doubling}. The distribution independent regret bound is given in Theorem \ref{doubling}.

    \begin{algorithm}
        \caption{Algorithm with unknown time horizon}
        \label{alg-doubling}
        \begin{algorithmic}[1]
            \For{$i=0,1,\dots$}
            \State In time horizon $2^i,2^i+1,\dots,2^{i+1}-1$, run Algorithm \ref{alg-elim} with time horizon $2^i$.
            \EndFor
        \end{algorithmic}
    \end{algorithm}

    \begin{restatable}{theorem}{doublingtrick}\label{doubling}
        The regret of Algorithm \ref{alg-doubling} is bounded by $20\sqrt{T\ln(KT^2)} + 3\log_2 T + 3$.
    \end{restatable}

    
    \subsection{Proof of Theorem \ref{indep}}\label{proof-indep}
    

    Because we want to observe as many arms as possible, we would like to choose an arm with a small index(a small position). 
    In this way, our algorithm maintains a set of arms $S_t$ in each round $t$, which is the set of arms that are possible to be the best arm. 
    We could let $S_t = [K]$ for each round, then this will lead to large regret, so we would like all arms in $S_t$ have means `close' to the mean of the best arm, and the best arm $i^*$ is in the set $S_t$. 
    In this way, we will define ``a procedure is nice at round $t$'' in Definition \ref{procedure} to describe the event that the best arm is in $S_t$ and all of the arms in $S_t$ have means close to that of the best arm. 
    Then we will show in Lemma \ref{hpe} that the procedure is nice at all rounds $t\le T$ with high probability. 
    Finally, we will use this lemma to prove Theorem \ref{indep}.
    To begin with, we have the following definition and a simple lemma.
    
    \begin{definition}\label{sampling}
        We call the sampling is nice at the beginning of round $t$ if $|\hat\mu_{i,t-1} - \mu_i| < \rho_t,\forall i \in S_{t-1}$, where $\rho_t = \sqrt{\frac{\ln (KT^2)}{2(t-1)}},\forall t\ge 2$ and $\rho_1 = \infty$. Let $\mathcal N_t^s$ denote this event.
    \end{definition}
    
    \begin{restatable}{lemma}{samplingnice}\label{sampling-nice-lemma}
        For each round $t\ge 1$, $\prob\{\lnot \mathcal N_t^s\} \le \frac{2}{T^2}$.
    \end{restatable}
    
    The proof of this lemma is simple with an application of the Hoeffding's Inequality followed by a union bound. \OnlyInFull{For more detail, please see Appendix \ref{sampling-nice-lemma-pf}.}\OnlyInShort{For more detail, please see Appendix.}
    Then, we have the definition for ``procedure is nice at round $t$'' and the main lemma that shows that the procedure is nice happens uniformly at all rounds with high probability. The formal definition is shown in Definition \ref{procedure} and the lemma is formally stated in Lemma \ref{hpe}.
    
    \begin{definition}\label{procedure}
        We say that the procedure is nice during the algorithm at round $t$ if both of the following are satisfied,
        \begin{enumerate}
            \item $i^*\in S_t$, where $i^* = \arg\max_{i\in [K]} \mu_i$.
            \item $\forall i\in S_t, \mu_{i^*} - \mu_{i} \le 4\rho_t$.
        \end{enumerate}
        Let $\mathcal N_t^p$ denote this event.
    \end{definition}
    \vspace{-2mm}
    \begin{lemma}\label{hpe}
        Let $\cM_t = \bigcap_{s=1}^t \mathcal N_s^p$, then 
        \vspace{-2mm}
        \[\forall t\in [T], \prob\{\lnot \cM_t\} \le \frac{2}{T}.\]
    \end{lemma}
    
    \begin{proof}
        We partition the event $\lnot \cM_t$ into disjoint events, we have
        \[\lnot\cM_t = \lnot \mathcal N_1^p \cup (\cM_1\cap\lnot \mathcal N_2^p) \cup \cdots \cup (\cM_{t-1}\cap \lnot \mathcal N_t^p).\]
        \vspace{-2mm}
        Note that $\lnot\cM_t$ is the union of disjoint events, so we have
        \[\prob\{\lnot\cM_t\} = \prob\{\lnot \mathcal N_1^p\} + \sum_{s=2}^t \prob\{\cM_{s-1}\cap \lnot\mathcal N_s^p\}.\vspace{-2mm}\]
        First, it is obvious that $\prob\{\lnot \mathcal N_1^p\} = 0$, since $\mathcal N_1^p$ will always happen, then we just need to bound $\prob\{\cM_{s-1}\cap \lnot\mathcal N_s^p\}$ for each $2\le s\le t$. We have
        \vspace{-2mm}
        \begin{align*}
        \prob\{\cM_{s-1}\cap \lnot\mathcal N_s^p\} =& \prob\left\{\left(\bigcap_{r=1}^{s-1}\mathcal N_r^p\right) \cap \lnot\mathcal N_s^p\right\} \\
        \le& \prob\{\mathcal N_{s-1}^p \cap \lnot\mathcal N_s^p\}.
        \vspace{-2mm}
        \end{align*}
        Then we prove that $\mathcal N_{s-1}^p \cap \lnot\mathcal N_s^p \Rightarrow \lnot\mathcal N_s^s$. In fact, if $\mathcal N_{s-1}^p$ happens, then we have $i^* \in S_{s-1}$, if $i^* \notin S_s$, then let $m_{s-1} = \arg\max_{i} \hat\mu_{i,s-1}$, we have
        \vspace{-2mm}
        \begin{align*}
        \vspace{-2mm}
        \mu_{i^*} - \hat\mu_{i^*,s-1} \ge& \mu_{m_{s-1}} - \hat\mu_{i^*,s-1} \\
        \ge& \mu_{m_{s-1}} - \hat\mu_{m_{s-1},s-1} + 2\rho_s,
        \vspace{-2mm}
        \end{align*}
        which leads to $\lnot \mathcal N_{s}^s$, since either $\mu_{i^*} - \hat\mu_{i^*,s-1} \ge \rho_s$ or $-\mu_{m_{s-1}} + \hat\mu_{m_{s-1},s-1} \ge \rho_s$ must happen. If $\mathcal N_{s-1}^p$ and $i^*\in S_s$ happens but $\exists i\in S_s,\mu_{i^*} - \mu_{i} > 4\rho_s$, then
        \vspace{-2mm}
        \begin{align*}
        &\mu_{i^*} - \hat\mu_{i^*,s-1} + \hat\mu_{i,s-1} - \mu_{i} \\
        \ge& 4\rho_{s} - \hat\mu_{i^*,s-1} + \hat\mu_{i,s-1}\\
        \ge& 4\rho_s - 2\rho_s\\
        =& 2\rho_s,
        \end{align*}
        which also leads to $\lnot\mathcal N_s^s$ by the same argument. So we have $\mathcal N_{s-1}^p \cap \lnot\mathcal N_s^p \Rightarrow \lnot\mathcal N_s^s$, then we have $\prob\{\mathcal N_{s-1}^p \cap \lnot\mathcal N_s^p\} \le P(\lnot\mathcal N_s^s) \le \frac{2}{T^2}$ from the previous lemma,
        \vspace{-2mm}
        \begin{align*}
        \prob\{\lnot\cM_t\} \le& \prob\{\lnot \mathcal N_1^p\} + \sum_{s=2}^t \prob\{\cM_{s-1}\cap \lnot\mathcal N_s^p\} \\
        \le& 0 + (t-1)\frac{2}{T^2} \\
        \le& \frac{2}{T}.
        \end{align*}\vspace{-2mm}\qed
        \vspace{-2mm}
    \end{proof}
    
    With the result of the previous lemma, we can prove Theorem \ref{indep}. The proof is just a combination of Lemma \ref{hpe} and direct calculation. We first partition the regret by an event $\cM_T = \bigcap_{j=1}^T \mathcal N_j^p$, which is defined in Lemma \ref{hpe}, representing the event that for all $t\le T$, the procedure is nice at round $t$. From Lemma \ref{hpe}, we know that the event will happen with high probability, and the regret in this case can be bounded easily. Then we just relax the regret in the case that $\cM_T$ does not happen to the worst case and we will complete the proof. \OnlyInFull{The proof of the theorem is straight forward, and we put the proof details in Appendix \ref{indep-proof}.}\OnlyInShort{The proof of the theorem is straight forward, and we put the proof details in Appendix.}

    With Theorem \ref{indep}, we can prove the regret for Algorithm \ref{alg-doubling}. Direct computation will lead to Theorem \ref{doubling}. \OnlyInFull{The detailed proof is shown in Appendix \ref{doubling-proof}.}\OnlyInShort{The detailed proof is shown in Appendix.}
    
    \subsection{Proof of Theorem \ref{dep}}
    
    
The proof of Theorem~\ref{dep} is based on the following key observation.
If arm $j$ has mean value larger than that of arm $j+1$, i.e. $\mu_j \ge \mu_{j+1}$ and $\Delta_j \le \Delta_{j+1}$, our algorithm will first play arm $j$ and find that arm $j+1$ is bad and eliminate arm $j+1$. Then it will play arm $j$ until arm $j$ is eliminated by the algorithm. However, if we exchange arm $j$ and arm $j+1$ such that in this case, $\mu_j < \mu_{j+1}$ and $\Delta_j < \Delta_{j+1}$, our algorithm will first play arm $j$ for several times and find that arm $j$ is bad and eliminate $j$, and then play arm $j+1$ until arm $j+1$ is eliminated. The number of total observations of arm $j$ and arm $j+1$ is the same, 
    but the regret of algorithm in the case of $\Delta_j < \Delta_{j+1}$ is worse then the case of
        $\Delta_j > \Delta_{j+1}$, because we spend more time playing the worse arm $j$ in the first case.
Therefore, the best sequence for our algorithms is 
    $\Delta_1 \le \Delta_2 \le \cdots \le \Delta_K$ with no regret, and the worst sequence is
    $\Delta_1 \ge \Delta_2 \ge \cdots \ge \Delta_K$. 
Similarly, if $i^*$ is the index of the best arm, when its index is fixed, for any sequence of arms before
    $i^*$, we can apply the above idea to do a bubble-sort on $\Delta_j$'s to change it into the worst sequence
    $\Delta_{1} \ge \Delta_2 \ge \cdots \ge \Delta_{i^*}$, and then use this worst sequence to bound the 
    regret.
In the following proof, we apply this bubble-sort idea to the proof of Lemma~\ref{opt}, which provides an 
    upper bound to the optimal solution of a linear integer program.
Then in the proof of Theorem~\ref{dep}, we show that the distribution-dependent regret is upper bounded
    by the optimal solution of the linear integer program.
    
%
%
    
    \begin{lemma}\label{opt}
        Let $\{\Delta_{(i)}\}$ be a permutation of $\{\Delta_i \mid i\le i^*\}$ such that $\Delta_{(1)} \ge \Delta_{(2)} \ge \dots \ge \Delta_{(i^*)} = 0$, 
        and let $C$ be a constant. Then, let $(a_1,\dots, a_{i^*}) \in \mathbb N^{i^*}$ denote the variables in the following optimization problem, the optimal value of the following optimization problem
        \begin{align*}
        \max_{(a_1,\dots, a_{i^*}) \in \mathbb N^{i^*}} & \sum_{j=1}^{i^*} a_j\Delta_j\\
        \text{s.t.} & \sum_{i=1}^j a_i \le \frac{C}{\Delta_j^2} + 1,\forall j\in\{j'| a_{j'} > 0,j\neq i^*\},
        \end{align*}
        is upper bounded by
        \[\Delta_{(1)} + \frac{C}{\Delta_{(1)}} + C\sum_{i=2}^{i^*-1}\left(\frac{1}{\Delta^2_{(i)}} - \frac{1}{\Delta^2_{(i-1)}}\right)\Delta_{(i)}.\]
    \end{lemma}
    
    \begin{proof}
        Let $\OPT$ denote the optimal value of the original optimization problem
        \begin{align}
        \max_{(a_1,\dots, a_{i^*}) \in \mathbb N^{i^*}} & \sum_{j=1}^{i^*} a_j\Delta_j\label{opt1}\\
        \text{s.t.} & \sum_{i=1}^j a_i \le \frac{C}{\Delta_j^2} + 1,\forall j\in\{j'| a_{j'} > 0,j\neq i^*\},\nonumber
        \end{align}
        and $\OPT'$ denote the optimal value of the modified optimization problem
        \begin{align}
        \max_{(a_1,\dots, a_{i^*}) \in \mathbb N^{i^*}} & \sum_{j=1}^{i^*} a_j\Delta_{(j)}\label{opt2}\\
        \text{s.t.} & \sum_{i=1}^j a_i \le \frac{C}{\Delta_{(j)}^2} + 1,\forall j\in\{j'| a_{j'} > 0,j\neq i^*\},\nonumber
        \end{align}
        where $\Delta_{(1)} \ge \Delta_{(2)} \ge \cdots \ge \Delta_{i^*}$ is a permutation of $\{\Delta_i\}_{i\le i^*}$.
        We first show that $\OPT \le \OPT'$. Suppose $\Delta_{j_0} < \Delta_{j_0+1}$. Let $\bar\Delta_{j_0} = \Delta_{j_0+1},\bar\Delta_{j_0+1} = \Delta_{j_0}$, and for all $k\neq j_0,j_0+1$, $\bar\Delta_k = \Delta_k$,i.e. $\{\bar\Delta_j\}$ is obtained by exchanging 2 adjacent elements in $\{\Delta_j\}$. Let $\overline{\OPT}$ denote the optimal value of the following optimization problem
        \begin{align}
        \max_{(a_1,\dots, a_{i^*}) \in \mathbb N^{i^*}} & \sum_{j=1}^{i^*} a_j\bar\Delta_{j}\label{opt3}\\
        \text{s.t.} & \sum_{i=1}^j a_i \le \frac{C}{\bar\Delta_{j}^{2}} + 1,\forall j\in\{j'| a_{j'} > 0,j\neq i^*\}.\nonumber
        \end{align}
        We just have to show that $\OPT \le \overline{\OPT}$, then $\OPT \le \OPT'$ can be obtained by repeatly exchanging 2 adjacent elements. To prove $\OPT \le \overline{\OPT}$, we just have to show that every feasible solution in the original optimization problem (\ref{opt1}) can be transformed into a feasible solution of the optimization problem (\ref{opt3}), with the same objective value.
        
        Let $x_1,\dots,x_{i^*}$ be any feasible solution of the original optimization problem (\ref{opt1}). Let $\bar x_j = x_j,\forall j\neq j_0,j_0+1$, and let $\bar x_{j_0} = x_{j_0+1},\bar x_{j_0+1} = x_{j_0}$, and it is obvious that the objective value in the optimization problem (\ref{opt1}) and (\ref{opt3}) are the same, since we exchange the coefficient and the variable at $j_0$ and $j_0+1$ at the same time. Then we show that $\bar x_1,\dots,\bar x_{i^*}$ is also a feasible solution in optimization problem (\ref{opt3}). 
        
        First for all $j \neq j_0,j_0+1$, we have $\sum_{i=1}^j \bar x_i \le \frac{C}{\bar\Delta_{j}^{2}} + 1$, since it is equivalent to $\sum_{i=1}^j x_i \le \frac{C}{\Delta_j^2} + 1$ and $\bar x_j > 0$ is equivalent to $x_j > 0$. 
        
%

        Then we consider the variable $\bar x_{j_0} = x_{j_0+1}$. 
        If $x_{j_0+1} > 0$, we have
        \[\sum_{i=1}^{j_0} \bar x_i \le \sum_{i=1}^{j_0+1} \bar x_i = \sum_{i=1}^{j_0+1} x_i \le \frac{C}{\Delta_{j_0+1}^{2}} + 1 = \frac{C}{\bar\Delta_{j_0}^{2}} + 1.\]
        If  $x_{j_0+1} = 0$, then $\bar x_{j_0} = 0$ and we do not have a constraint for
        $j = j_0 $ in problem (\ref{opt3}). 
        
        Next we consider the variable $\bar x_{j_0+1} = x_{j_0}$.
        If $x_{j_0+1} > 0$, using
            $\bar\Delta_{j_0+1} < \bar\Delta_{j_0} $ we have
        \[\sum_{i=1}^{j_0+1} \bar x_i = \sum_{i=1}^{j_0+1} x_i \le  \frac{C}{\Delta_{j_0+1}^{2}} + 1 = \frac{C}{\bar\Delta_{j_0}^{2}} + 1 \le \frac{C}{\bar\Delta_{j_0+1}^{2}} + 1.\]
        If $x_{j_0+1} = 0$ and $x_{j_0} > 0$, we have
        \[\sum_{i=1}^{j_0+1} \bar x_i = \sum_{i=1}^{j_0+1} x_i 
        = \sum_{i=1}^{j_0} x_i 
        \le  \frac{C}{\Delta_{j_0}^{2}} + 1 =  \frac{C}{\bar\Delta_{j_0+1}^{2}} + 1.\]
        If $x_{j_0} = 0$, then $\bar x_{j_0+1} = 0$ and we do not need a constraint for $j= j_0 + 1$
        in problem (\ref{opt3}). 
        
%
%
Therefore, after discussing all cases, we know that
    $(\bar x_1,\dots,\bar x_K)$ is a feasible solution of the optimization problem (\ref{opt3}). 
Then with our previous argument, the optimal value $\OPT'$ of optimization problem (\ref{opt2}) is at least $\OPT$, i.e. $\OPT \le \OPT'$.
        
        Then suppose $\{x_{r_i}\}$ is a feasible solution of the modified optimization problem (\ref{opt2}), we have
        \begin{align*}
        \sum_{i=1}^{i^*} x_{r_i}\Delta_{(i)}
        =& x_{r_1}\Delta_{(1)} + \sum_{i=2}^{i^*}\left(\sum_{j=1}^i x_{r_j} - \sum_{j=1}^{i-1} x_{r_j}\right)\Delta_{(i)}\\
        =&\sum_{i=1}^{i^*-1}\left(\Delta_{(i)} - \Delta_{(i+1)}\right)\sum_{j=1}^i x_{r_j} + \Delta_{(i^*)}\sum_{j=1}^{i^*} x_{r_j}\\
        \le&\sum_{i=1}^{i^*-1}\left(\Delta_{(i)} - \Delta_{(i+1)}\right)\left(\frac{C}{\Delta_{(i)}^2} + 1\right) \\
        =& \Delta_{(1)} + \frac{C}{\Delta_{(1)}} + C\sum_{i=2}^{i^*-1}\left(\frac{1}{\Delta^2_{(i)}} - \frac{1}{\Delta^2_{(i-1)}}\right)\Delta_{(i)},
        \end{align*}
        where we use the fact that $\Delta_{(i^*)} = 0$. So $\OPT'$ is also upper bounded, which complete the proof directly.\qed
    \end{proof}
    
    With the conclusion of the lemma, we can prove Theorem \ref{dep}. 
    The general idea to prove Theorem \ref{dep} is the same as proving Theorem \ref{indep}. We first partition the regret by the event $\cM_T$, which is defined in Definition \ref{procedure}. With Lemma \ref{hpe}, $\cM_T$ will happen with high probability, and we can just consider the regret when $\cM_T$ happens. Then we bound the regret when $\cM_T$ happens from the help of Lemma \ref{opt}.
    
    \begin{proof}[Proof of Theorem \ref{dep}]
        Similar to the proof of Theorem \ref{indep}, we have
        \begin{align*}
        \E[\sum_{t=1}^T (\mu_{i^*} - \mu_{I_t})] \le& \E[\sum_{t=1}^T (\mu_{i^*} - \mu_{I_t})|\cM_T] +\E[\sum_{t=1}^T (\mu_{i^*} - \mu_{I_t})|\lnot \cM_T]\cdot \prob\{\lnot\cM_T\},
        \end{align*}
        and
        \[\E[\sum_{t=1}^T (\mu_{i^*} - \mu_{I_t})|\lnot \cM_T] \le T,\prob\{\lnot\cM_T\} \le \frac{2}{T}.\]
        Suppose that the arms $1,2,\dots,K$ are played for $a_1,a_2,\dots,a_K$ times after $T$ rounds and $\cM_T$ happens, then the regret is $\sum_{i=1}^Ka_i \Delta_i$. Then we show that when $\cM_T$ happens,
        \[\sum_{i=1}^K a_i \Delta_i \le \Delta_{(1)} + \frac{C}{\Delta_{(1)}} + C\sum_{i=2}^{i^*-1}\left(\frac{1}{\Delta^2_{(i)}} - \frac{1}{\Delta^2_{(i-1)}}\right)\Delta_{(i)}.\vspace{-2mm}\]
        First, we only have to consider the arm $j$ with $j < i^*$, since if $\cM_T$ happens, our elimination based algorithm (see Algorithm \ref{alg-elim}) will never choose arm $j>i$ to play, and for arm $i^*$ there is no regret contribution. For arm $j < i^*$, if $a_j \neq 0$, then at the last time the algorithm plays arm $j$, arm $j$ has been observed for $\sum_{i=1}^j a_j -1$ times, since we only delete the arms in set $S$ so $I_t$ must be non-decreasing. Then as $\cM_T$ happens, we have
        \[\Delta_j \le 4\sqrt{\frac{\ln(KT^2)}{2(\sum_{i=1}^j a_i -1)}},\vspace{-2mm}\]
        which will lead to
        \[\sum_{i=1}^j a_i \le \frac{C}{\Delta_j^2} + 1,\]
        where $C = 8\ln(KT^2)$ as defined in Theorem \ref{dep}. Then we can conclude that when $\cM_T$ happens, the regret is bounded by
        \begin{align*}
        \max_{(a_1,\dots, a_{i^*}) \in \mathbb N^{i^*}} & \sum_{j=1}^{i^*-1} a_j\Delta_j\\
        \text{s.t.} & \sum_{i=1}^j a_i \le \frac{C}{\Delta_j^2} + 1,\forall j\in\{j'| a_{j'} > 0,j\neq i^*\}.
        \end{align*}
        Then from Lemma \ref{opt}, we know that the optimal value of the above optimization problem is upper bounded, so we have
        \begin{align*}
        \E[\sum_{t=1}^T (\mu_{i^*} - \mu_{I_t})|\cM_T] 
        &\le \Delta_{(1)} + \frac{C}{\Delta_{(1)}} + C\sum_{i=2}^{i^*-1}\left(\frac{1}{\Delta^2_{(i)}} - \frac{1}{\Delta^2_{(i-1)}}\right)\Delta_{(i)}
        \end{align*}
        Then combine with the previous result, we can finish the proof.\qed
    \end{proof}
    
    Then we have a corollary from this distribution dependent bound.
    
    \begin{corollary}
        Let $\{\Delta_{(i)}\}$ be a permutation of $\{\Delta_i\}$ such that $\Delta_{(1)} \ge \Delta_{(2)} \ge \dots \ge \Delta_{(i^*)} = 0$, and $C = 8\ln(KT^2)$
         as defined in Theorem \ref{dep}, then the regret is bounded by $\left(\frac{C}{\Delta_{(i^*-1)}^2}+1\right)\Delta_{(1)} + 2$.
    \end{corollary}
    
    \section{Lower Bound}\label{sec-low}
    The lower bound for multi-armed bandit problems has been extensively studied. 
    However, we notice that there is no regret lower bound for the full-information multi-armed bandit under the stochastic case. 
    In this section, we show that the regret is lower bounded by $\Omega(\sqrt{T\log K})$ in this case, 
    which also implies a regret lower bound of $\Omega(\sqrt{T\log K})$ for the one-sided bandit case.
    Comparing with the regret upper bound of $O(\sqrt{T(\log K + \log T)})$ of Theorem~\ref{indep},
    we can see that our elimination algorithm gives almost a tight regret bound.
    
    In this section, we fix a bandit algorithm. Let $I_t$ denote the choice of the algorithm in round $t$. Let $K$ denote the total number of arms. For each $j\in[K]$, let $\mathcal I_j$ denote the problem instance that $\mu_{k} = \frac{1}{2}$, for all $k\neq j$, $\mu_{j} = \frac{1+\varepsilon}{2}$ for some small $\varepsilon>0$, and each arm is a Bernoulli random variable independent from other arms. 
    
    The proof follows from the original proof of lower bound for bandit feedback MAB problem \cite{BCB12}, but we need more careful calculation. 
    The original proof for the bandit feedback regret lower bound is $\sqrt{TK}$, and if we directly apply it to the full information feedback case, we would get $\sqrt{T}$ lower bound.
    With more careful analysis, we could raise this lower bound to $\sqrt{T\log K}$. Following the original analysis, we connect the full information MAB problem with the {\em bandit-with-prediction} problem, 
        in which the algorithm is given the rewards of all arms in the first $T$ rounds, and it needs to decide which is the best arm. 
    We use $y_T$ to denote the output of an algorithm of the bandit-with-prediction problem in this section. 
    Naturally, we can select the arm with the largest cumulative rewards in the first $T$ rounds as $y_T$, and this is
        called Follow-the-Leader strategy.
    Then we use the reverse Chernoff Bound (Lemma \ref{tight-chernoff}) to show the regret lower bound for 
    the Follow-the-Leader strategy, and then we show that Follow-the-Leader strategy has the optimal regret among all the algorithm (up to constants).
    Finally, we reduce the full information MAB problem to the bandit-with-prediction problem to show its lower bound.

%

\begin{restatable}{lemma}{tightchernoff}{\rm\bf (Tightness of Chernoff Bound)} \label{tight-chernoff}
        Suppose $X_1,X_2,\dots,X_n$ are i.i.d Bernoulli random variable with $\prob[X_1 = 1] = \frac{1}{2}$, then there exists absolute constants $c',d,p$ such that for all $0 < \varepsilon <d$ such that $\varepsilon^2 \cdot n > p$,
        \[\prob\left\{\frac{1}{n}\sum_{i=1}^n X_i > \frac{1}{2} + \varepsilon \right\} > e^{-c'n\varepsilon^2}.\]
\end{restatable}

    The above lemma is a well-known result. For convenience, we put the proof of this lemma in the appendix. 
\OnlyInFull{See Appendix \ref{lbproof} for more details.
}\OnlyInShort{Please see Appendix for more details.} The following lemma shows that the Follow-the-Leader strategy still could make mistakes on the bandit-with-prediction task.
    
    \begin{restatable}{lemma}{FTLmistake}
        Suppose $\frac{c\ln K}{2\varepsilon^2} \le T \le \frac{c\ln K}{\varepsilon^2}$, for a small enough absolute constant $c$ (which is not the constant in the previous lemma) and $0\le \varepsilon < d$ (where $d$ is the absolute constant in the previous lemma). Consider the algorithm Follow-the-Leader for 
        the bandits-with-prediction problem.
        Then for large enough $K$,
        \[\sum_{j=1}^K\prob\{y_T = j|\mathcal I_j\} \le \frac{K}{4}.\]
    \end{restatable}
    
    The next lemma shows that no other algorithms can do much better than the Follow-the-Leader strategy, for the
    bandit-with-prediction problem.
    
    \begin{restatable}{lemma}{FTLbest} \label{FTL-best}
        Suppose $\frac{c\ln K}{2\varepsilon^2} \le T \le \frac{c\ln K}{\varepsilon^2}$, for a small enough absolute constant $c$, a large enough $K$ and $0\le \varepsilon < d$ where $d$ is the constant in previous lemma. Then for any (deterministic or
        randomized) algorithm for the bandit-with-prediction problem, there exists at least $\lceil K/3 \rceil$ arms $j$ such that
        \[\prob\{y_T = j|\mathcal I_j\} \le \frac{3}{4}.\]
    \end{restatable}

    We can now prove the regret lower bound of the full information bandit problem by utilizing the above result for the bandit-with-prediction problem.
    
    \begin{restatable}{theorem}{THMlowerbound}{\rm\bf (Regret lower bound for full information stochastic bandits)}
        Fix time horizon $T$ and the number of arms $K$ such that $\sqrt{c\ln K / T} < d$, where $c,d$ are the constants in Lemma \ref{FTL-best}.
         When $K$ is big enough, then for any bandit algorithm, there exists a problem instance such that $\E[R(T)] \ge \Omega(\sqrt{T\log K})$.
    \end{restatable}
    \OnlyInFull{Please see Appendix \ref{lbproof} for the missing proofs.}\OnlyInShort{Please see Appendix for the missing proofs.}
    
    \section{Numerical Experiments}\label{sec-exp}
    In this section, we show numerical experiments on our elimination based algorithm ELIM
    together with two other algorithms:
    (a) EXP3-RTB algorithm introduced in \cite{DBLP:conf/colt/Cesa-BianchiGGG17}, 
    which solves one-sided full information bandit in the adversarial case, 
    and
    (b) UCB-N algorithm introduced in \cite{CaronKLB12} to solve stochastic multi-armed
    bandit with side information, and it is essentially UCB but updates any arm
    when it has an observation, not just the arm played in the round.
    
    First, we do experiments when all the suboptimal arms have the same mean of $0.6$
        with a gap $\Delta$ towards the best arm,
        similar to our lower bound analysis setting.
    We will show results with different $\Delta$ setting (Fig.\ref{fig:non-random-delta}) and different best arm position (Fig.\ref{fig:non-random-best}). 
    For convenience, we let the reward of each arm follows a Bernoulli distribution.
    Next, we do experiments when the suboptimal arms have means drawn uniformly
    at random from $(0.2,0.6)$ 
    except for the mean of the best arm, which is set to $0.6 + \Lambda$ for a parameter $\Lambda$. 
    We vary the value of $\Lambda$ (Fig.\ref{fig:random-delta}) and the position of the best arms (Fig.\ref{fig:random-best}). 
    
    We use $T$ to denote the total time horizon we choose in the experiments. In most of the experiments, we choose $T = 100000$, but we will choose $T = 200000$ to better distinguish the performance between different algorithms in some cases. We use $K$ to denote the number of arms in our experiments, and we choose $K = 20$ in all of the experiments. We use $Best$ to denote the position of the best arm, which is set to $3,10,17$ in different experiments. 
 For each experiment, we run 100 times and draw the 99\% confidence interval
     surrounding the curve (all are very narrow regions surrounding the curve).
    
    \begin{figure}[t]
        \centering
        \subfigure[$\Delta = 0.1$]{\includegraphics[width=1.5in]{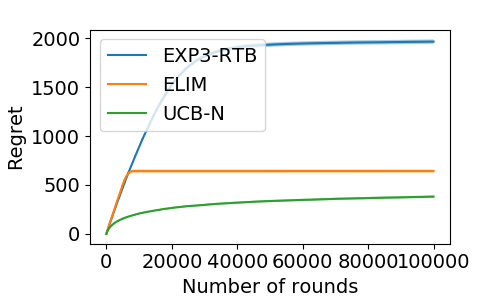}}
        \subfigure[$\Delta = 0.05$]{\includegraphics[width=1.5in]{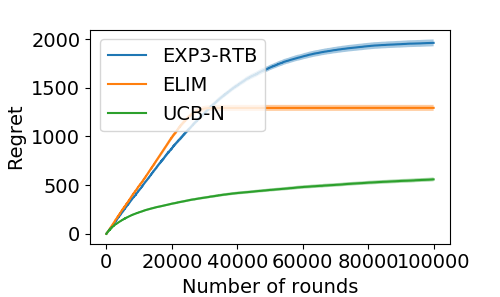}}
        \subfigure[$\Delta = 0.03$]{\includegraphics[width=1.5in]{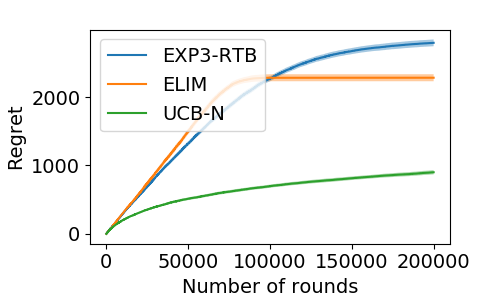}}
        \caption{Uniform-mean suboptimal arms with  $K = 20$, $Best = 17$, and
            varying $\Delta$.
        }
        \label{fig:non-random-delta}
    \end{figure}
    
    \begin{figure}[t]
        \centering
        \subfigure[$Best = 10$]{\includegraphics[width=1.5in]{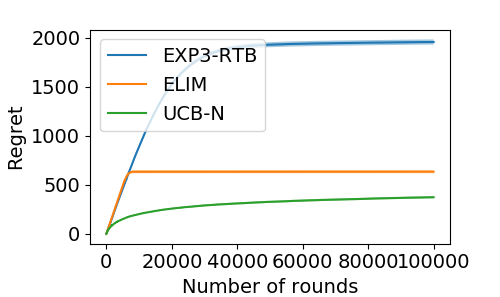}}
        \subfigure[$Best = 3$]{\includegraphics[width=1.5in]{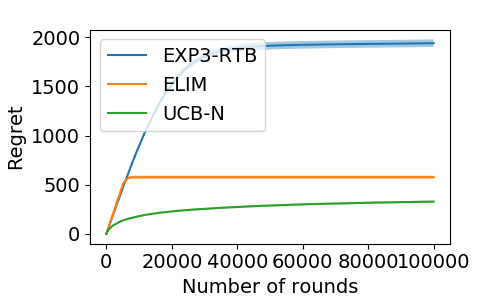}}
        \caption{
            Uniform-mean suboptimal arms with  $K = 20$, $\Delta = 0.1$, and
            varying $Best$.
        }
        \label{fig:non-random-best}
    \end{figure}
    
    \begin{figure}[t]
        \centering
        \subfigure[$\Lambda = 0.1$]{\includegraphics[width=1.5in]{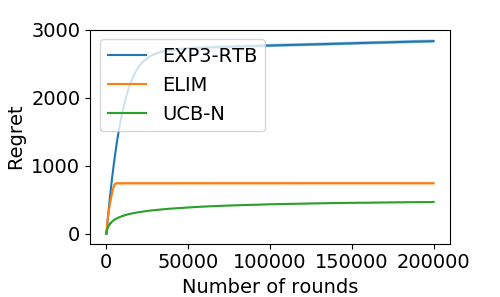}}
        \subfigure[$\Lambda = 0.05$]{\includegraphics[width=1.5in]{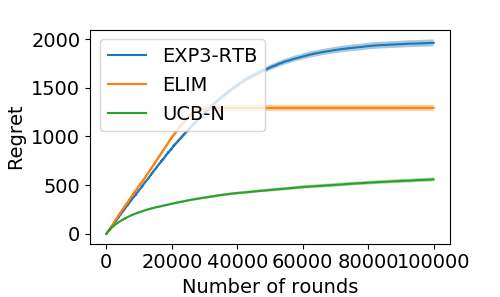}}
        \subfigure[$\Lambda = 0.03$]{\includegraphics[width=1.5in]{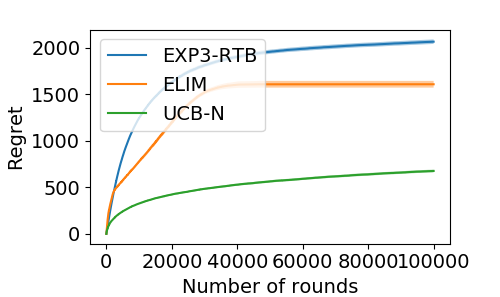}}
        \caption{
            Random-mean suboptimal arms with  $K = 20$, $Best = 17$, and
            varying $\Delta$.
        }
        \label{fig:random-delta}
    \end{figure}
    
    \begin{figure}[t]
        \centering
        \subfigure[$Best=10$]{\includegraphics[width=1.5in]{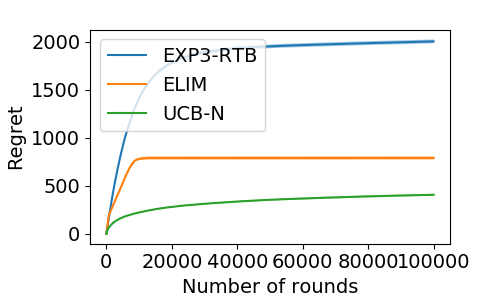}}
        \subfigure[$Best=3$]{\includegraphics[width=1.5in]{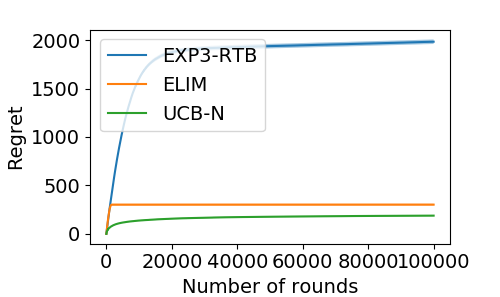}}
        \caption{
            Uniform-mean suboptimal arms with  $K = 20$, $\Delta = 0.1$, and
            varying $Best$.
        }
        \label{fig:random-best}
    \end{figure}
    
    From the above experiments, we can find that
    \begin{enumerate}
        \item In both experiments, when the gap between the mean of the best arm and the mean of others is larger, our algorithm performs much better than the existing EXP3-RTB algorithm.
        \item In the first experiments, when we change the position of the best arm, the regret line does not change so much. In the second experiments where we add more randomness, if the position of the best arm has small index, then our algorithm will perform better. However, the existing EXP3-RTB algorithm does not have this property.
        \item UCB-N consistently outperforms both our algorithm ELIM and the
        EXP3-RTB algorithm.
    \end{enumerate}

Therefore, we can conclude in the stochastic setting, our elimination-based algorithm
    performs much better than the EXP3-RTB algorithm designed for the same problem but on
    the adversarial setting, and UCB-N has the best empirical performance.
The issue with UCB-N is that we cannot derive a tight theoretical regret bound that
    also beats or even match ELIM.
If we simply use UCB regret bound for UCB-N, it would be too loose and it would be
    inferior to our elimination based algorithm, as discussed after
    Theorem~\ref{dep}.
The result in~\cite{CaronKLB12} on UCB-N cannot be applied here either because it requires
    mutually observable cliques in the observation graph but for the one-sided
    full-information case, the only cliques are the trivial singletons, which
    makes their regret bound reduced to the UCB regret bound.
Therefore, our algorithm ELIM is the one that achieves the best theoretical
    regret bound, significantly outperform the EXP3-RTB algorithm for the adversarial case,
    while UCB-N has the best empirical performance with an unknown tight theoretical guarantee.

    \section{Conclusion and Further Work}
    
    In this paper, we study the stochastic one-sided full-information bandit and propose an elimination-based algorithm to solve the problem.
    We provide the upper bounds of the algorithm, and show that it almost matches the lower bound of the problem.
    Our experiment demonstrates that it performs better than the algorithm designed for the adversarial setting.
    To the best of our knowledge, our algorithm achieves the best regret bound so far.
    
    One open problem is definitely on the analysis of UCB-N. As we have discussed, its naive regret bound such as the UCB regret bound would be much worse than our elimination algorithm, but its empirical performance shows better results. 
    We are trying to provide a tighter analysis on UCB-N, but it evades several attempts we have made so far, and thus we left it as a future research question. Another direction is to design other algorithms that better utilizes the one-sided full-information feedback
    structure and achieves both good theoretical and empirical results.
    Other specific feedback structures corresponding to practical applications are also worth further investigation.
    
    \section{Acknowledgement}
    
    	Wei Chen is partially supported by 
    the National Natural Science Foundation of China (Grant No. 61433014).

%
%
%
\bibliographystyle{splncs04}
\bibliography{papref}

\OnlyInFull{

\clearpage
\appendix
\section*{Supplementary Material}
This part contains the missing proofs in the main text. For convenience, we restate the theorems and lemmas here.

\section{Missing Proof in Section \ref{proof-indep}}
\subsection{Proof of Lemma \ref{sampling-nice-lemma}}\label{sampling-nice-lemma-pf}
{\samplingnice*}
	\begin{proof}[Proof of Lemma \ref{sampling-nice-lemma}]
		The lemma is correct in round $t=1$, so we can assume that $t\ge 2$.
		\begin{align*}
		\prob\{\lnot \mathcal N_t^s\} =& \prob\{\exists i\in S_{t-1},|\hat\mu_{i,t-1} - \mu_i| \ge \rho_t\} \\
		\le & \sum_{i\in S_{t-1}}\prob\{|\hat\mu_{i,t-1} - \mu_i| \ge \rho_t\} \\
		= & \sum_{i\in S_{t-1}}\prob\left\{|\hat\mu_{i,t-1} - \mu_i| \ge \sqrt{\frac{\ln (KT^2)}{2(t-1)}}\right\} \\
		\le & \sum_{i\in S_{t-1}}2\exp \left(-2(t-1)\left(\sqrt{\frac{\ln (KT^2)}{2(t-1)}}\right)^2 \right) \\
		=& \sum_{i\in S_{t-1}}\frac{2}{KT^2} \\
		\le & \frac{2}{T^2}.
		\end{align*}\qed
	\end{proof}

%
\subsection{Proof of Theorem \ref{indep}}\label{indep-proof}
{\distindepthm*}
\begin{proof}[Proof of Theorem \ref{indep}]
		The regret of the algorithm can be written as $\E[\sum_{t=1}^T (\mu_{i^*} - \mu_{I_t})]$. Then we have
		\begin{align*}
		\E\left[\sum_{t=1}^T (\mu_{i^*} - \mu_{I_t}) \right] 
		=& \E[\sum_{t=1}^T (\mu_{i^*} - \mu_{I_t})|\cM_T]\cdot \prob\{\cM_T\} \\
		&\quad +\E[\sum_{t=1}^T (\mu_{i^*} - \mu_{I_t})|\lnot \cM_T]\cdot \prob\{\lnot\cM_T\} \\
		\le& \E[\sum_{t=1}^T (\mu_{i^*} - \mu_{I_t})|\cM_T] +\E[\sum_{t=1}^T (\mu_{i^*} - \mu_{I_t})|\lnot \cM_T]\cdot \prob\{\lnot\cM_T\}
		\end{align*}
		We have
		\[\E[\sum_{t=1}^T (\mu_{i^*} - \mu_{I_t})|\lnot \cM_T] \le T,\prob\{\lnot\cM_T\} \le \frac{2}{T}.\]
		Then given $\cM_T$, we can conclude that $\forall t\le T,\mu_{i^*} - \mu_{I_t} \le 4\rho_t$. Besides, it is obvious that $\mu_{i^*} - \mu_{I_t} \le 1$, since we made the assumption that the reward has support $[0,1]$, then we have
		\begin{align*}
		\E[\sum_{t=1}^T (\mu_{i^*} - \mu_{I_t})|\cM_T]
		\le& 1 + \E[\sum_{t=2}^T (\mu_{i^*} - \mu_{I_t})|\cM_T] \\
		\le& 1 + \sum_{t=2}^T 4\rho_t\\
		=& 1 + 4\sqrt{\frac{\ln(KT^2)}{2}}\sum_{t=2}^T\frac{1}{\sqrt{t-1}} \\
		=& 1 + 4\sqrt{\frac{\ln(KT^2)}{2}}\left(1 + \sum_{t=2}^{T-1}\frac{1}{\sqrt{t}}\right)\\
		\le& 1 + 4\sqrt{\frac{\ln(KT^2)}{2}}\left(1 + \int_{t=1}^{T-1}\frac{1}{\sqrt{t}}dt\right)\\
		\le& 1 + 4\sqrt{\frac{\ln(KT^2)}{2}}\left(1 + 2\sqrt{x}|_{x=1}^{T-1}\right)\\
		\le& 1 + 4\sqrt{\frac{\ln(KT^2)}{2}}\cdot 2\sqrt{T}\\
		=& 1 + 4\sqrt{2T\ln(KT^2)}.
		\end{align*}
		And we can conclude that
		\begin{align*}
		\E[\sum_{t=1}^T (\mu_{i^*} - \mu_{I_t})]
		\le& \E[\sum_{t=1}^T (\mu_{i^*} - \mu_{I_t})|\cM_T] +\E[\sum_{t=1}^T (\mu_{i^*} - \mu_{I_t})|\lnot \cM_T]\cdot \prob\{\lnot\cM_T\} \\
		\le& 1 + 4\sqrt{2T\ln(KT^2)} +T \cdot \frac{2}{T}\\
		=& 4\sqrt{2T\ln(KT^2)} + 3.
		\end{align*}\qed
	\end{proof}
\subsection{Proof of Theorem \ref{doubling}}\label{doubling-proof}
{\doublingtrick*}
\begin{proof}[Proof of Theorem \ref{doubling}]
		First, it is obvious that, if we set the time horizon as $T$ in the algorithm, but we only play for $T' \le T$ rounds, then the regret is also bounded by $4\sqrt{2T\ln(KT^2)} + 3$. Then suppose that $2^k \le T < 2^{k+1}$, then the total regret is bounded by
		\begin{align*}
		\E[\sum_{t=1}^T (\mu_{i^*} - \mu_{I_t})]
		\le& \E[\sum_{t=1}^{2^{k+1}-1} (\mu_{i^*} - \mu_{I_t})] \\
		\le& \sum_{i=0}^k\E[\sum_{t=2^i}^{2^{i+1}-1}(\mu_{i^*} - \mu_{I_t})] \\
		\le& \sum_{i=0}^k (4\sqrt{2\cdot 2^i \ln(K\cdot 2^{2i})} + 3)\\
		\le& \sum_{i=0}^k 4\sqrt{2 \ln(K\cdot 2^{2k})}\cdot(\sqrt{2})^i + 3(k+1)\\
		=&  4\sqrt{2 \ln(K\cdot 2^{2k})}\frac{\sqrt{2^{k+1}}-1}{\sqrt{2}-1} + 3(k+1)\\
		\le&  4\sqrt{2 \ln(K\cdot T^2)}(\sqrt{2}+1)(\sqrt{2T}-1) + 3(\log_2 T + 1) \\
		\le& 20\sqrt{T\ln(K\cdot T^2)} + 3\log_2 T + 3
		\end{align*}
		where the last inequality comes from the fact that $\sqrt{2}+1<2.5$.\qed
	\end{proof}

\section{Lower Bound Proof}\label{lbproof}
{\tightchernoff*}

The following proof comes from the answer on Stackexchange(see \cite{14476} for the original proof). Before we show the proof of the tightness of Chernoff bound, we need some small claims.

\begin{proposition}[Stirling Approximation Corollary]\label{stirling-approx}
    If we have $1\le l\le n-1$, where $l,n\in\Z_+$, then we have
    \[\binom{n}{l} \ge \frac{1}{e\sqrt{2\pi l}}\left(\frac{n}{l}\right)^l\left(\frac{n}{n-l}\right)^{n-l}.\]
\end{proposition}

\begin{proof}
    From the Stirling's approximation(see \cite{robbins1955remark} for more detail), we have
    \[n! = \sqrt{2\pi n}\left(\frac{n}{e}\right)^ne^{\lambda},\]
    where $\frac{1}{12n+1}\le \lambda \le \frac{1}{12n}$.
    Since $\binom{n}{l} = \frac{n!}{l!}{(n-l)!}$, we have
    \begin{align*}
        \binom{n}{l} =& \frac{n!}{l!}{(n-l)!} \\
        \ge&\frac{\sqrt{2\pi n}\left(\frac{n}{e}\right)^n}{\sqrt{2\pi l}\left(\frac{l}{e}\right)^l\sqrt{2\pi {(n-l)}}\left(\frac{{n-l}}{e}\right)^{n-l}}\cdot \exp{\left(\frac{1}{12n+1} - \frac{1}{12l} - \frac{1}{12(n-l)}\right)} \\
        \ge& \frac{1}{e\sqrt{2\pi l}}\left(\frac{n}{l}\right)^l\left(\frac{n}{n-l}\right)^{n-l},
    \end{align*}
    where we use the fact that $\frac{1}{12n+1} - \frac{1}{12l} - \frac{1}{12(n-l)} \ge -1$.
    \qed
\end{proof}

\begin{proof}[Proof of Lemma \ref{tight-chernoff}]
We will show that, if $\varepsilon^2 n > 6$, then for any $0 < \varepsilon \le \frac{1}{2}$, we have
\[\prob\{\frac{1}{n}\sum_{i=1}^n X_i < \frac{1}{2}(1 - \varepsilon)\} \ge e^{-\frac{9}{2}n\varepsilon^2},\]
which will conclude the proof of Lemma \ref{tight-chernoff}.

We have
\[\prob\{\frac{1}{n}\sum_{i=1}^n X_i < \frac{1}{2}(1 - \varepsilon)\} = \frac{1}{2^n}\sum_{l=0}^{\lceil\frac{n}{2}(1 - \varepsilon)-1\rceil}\binom{n}{l}.\]
We fix $l_0 = \lfloor \frac{(1-2\varepsilon)n}{2}\rfloor+1$. Since the terms $\binom{n}{l}$ is increasing in terms of $l$ when $0 \le l \le \lfloor\frac{n}{2}(1 - \varepsilon)\rfloor$, we know that
\[\binom{n}{l} \ge \binom{n}{l_0},\forall l_0 \le l \le \lceil\frac{n}{2}(1 - \varepsilon)-1\rceil.\]
Then, we can lower bound the following term $\sum_{l=0}^{\lceil\frac{n}{2}(1 - \varepsilon)-1\rceil}\binom{n}{l}$ by
\begin{align*}
    \sum_{l=0}^{\lceil\frac{n}{2}(1 - \varepsilon)-1\rceil}\binom{n}{l} \ge& \sum_{l=0}^{l_0-1} 0 + \sum_{l=l_0}^{\lceil\frac{n}{2}(1 - \varepsilon)-1\rceil}\binom{n}{l_0} \\
    \ge& \left(\frac{n}{2}(1 - \varepsilon)-1 - \frac{(1-2\varepsilon)n}{2}\right)\binom{n}{l_0} \\
    =& \left(\frac{\varepsilon n}{2}-1\right)\binom{n}{l_0}.
\end{align*}
Then the probability can be bounded by
\begin{align*}
    \prob\{\frac{1}{n}\sum_{i=1}^n X_i < \frac{1}{2}(1 - \varepsilon)\} =& \frac{1}{2^n}\sum_{l=0}^{\lceil\frac{n}{2}(1 - \varepsilon)-1\rceil}\binom{n}{l}\\
    \ge& \frac{1}{2^n}\left(\frac{\varepsilon n}{2}-1\right)\binom{n}{l_0} \\
    \ge& \frac{1}{2^n}\left(\frac{\varepsilon n}{2}-1\right)\frac{1}{e\sqrt{2\pi l_0}}\left(\frac{n}{l_0}\right)^{l_0}\left(\frac{n}{n-l_0}\right)^{n-l_0}.
\end{align*}
Then we can show that $\prob\{\frac{1}{n}\sum_{i=1}^n X_i < \frac{1}{2}(1 - \varepsilon)\} > e^{-\frac{9}{2}n\varepsilon^2}$ by showing that
\[\left(\frac{\varepsilon n}{2}-1\right)\frac{1}{e\sqrt{2\pi l_0}} \ge e^{-\frac{1}{2}\varepsilon^2 n},\]
and
\[\frac{1}{2^n}\left(\frac{n}{l_0}\right)^{l_0}\left(\frac{n}{n-l_0}\right)^{n-l_0} \ge e^{-4\varepsilon^2 n}.\]
We first prove that $\left(\frac{\varepsilon n}{2}-1\right)\frac{1}{e\sqrt{2\pi l_0}} \ge e^{-\frac{1}{2}\varepsilon^2 n}$. Since we assume that $\varepsilon^2 n > 3$ and $\varepsilon \le \frac{1}{2}$, we have $\varepsilon n > 6$. Then we have $\frac{\varepsilon n}{2}-1 > \frac{2}{3}\frac{\varepsilon n}{2} = \frac{\varepsilon n}{3}$. We just have to prove that $\frac{\varepsilon n}{3}\frac{1}{e\sqrt{2\pi l_0}} \ge e^{-\frac{1}{2}\varepsilon^2 n}$. Since we assume that $\varepsilon^2 n > 6$, we have $e^{-\frac{1}{2}\varepsilon^2 n} < e^{-3} <0.04 $. As for the left hand side, we have $l_0 \le \frac{n}{2}$, and we have
\[\frac{\varepsilon n}{3}\frac{1}{e\sqrt{2\pi l_0}} \ge \frac{\varepsilon n}{3}\frac{1}{e\sqrt{n\pi}} = \frac{\varepsilon\sqrt{n}}{3e} \ge \frac{\sqrt{6}}{3e} > 0.1 > e^{-\frac{1}{2}\varepsilon^2 n}.\]
Then we prove that
\[\frac{1}{2^n}\left(\frac{n}{l_0}\right)^{l_0}\left(\frac{n}{n-l_0}\right)^{n-l_0} \ge e^{-4\varepsilon^2 n}.\]
Let $\delta$ takes the value such that $l_0 = (1-\delta)\frac{n}{2}$. From the definition of $l_0$, we have $\delta \le 2\varepsilon$, and it suffice to show that
\[\frac{1}{2^n}\left(\frac{n}{l_0}\right)^{l_0}\left(\frac{n}{n-l_0}\right)^{n-l_0} \ge e^{-\delta^2 n}.\]
The above inequality is equivalent to (by taking both sides to the $-\frac{1}{l_0}$ power)
\[2^{\frac{n}{l_0}}\cdot\frac{l_0}{n}\left(\frac{n-l_0}{n}\right)^{\frac{n-l_0}{l_0}} \le e^{\frac{\delta^2 n}{l_0}},\]
which is also equivalent to
\[\frac{2l_0}{n}\left(\frac{2(n-l_0)}{n}\right)^{\frac{n-l_0}{l_0}} \le e^{\frac{\delta^2 n}{l_0}}.\]
Substitute $l_0 = (1-\delta)\frac{n}{2}$ into the above inequality, we get the following equivalent inequality,
\[\frac{2(1-\delta)\frac{n}{2}}{n}\left(\frac{2(n-(1-\delta)\frac{n}{2})}{n}\right)^{\frac{n-(1-\delta)\frac{n}{2}}{(1-\delta)\frac{n}{2}}} \le e^{\frac{\delta^2 n}{(1-\delta)\frac{n}{2}}},\]
which can be simplified to the following form
\[(1-\delta)\left(1+\delta\right)^{\frac{2}{1-\delta}-1} \le e^{\frac{2\delta^2}{1-\delta}}.\]
Taking the logarithm on both sides, and use the inequality $\ln z\le z$, we have
\begin{align*}
    \ln\text{LHS} =& \ln(1-\delta) + \left(\frac{2}{1-\delta}-1\right)\ln\left(1+\delta\right) \\
    \le& -\delta + \left(\frac{2}{1-\delta}-1\right)\delta \\
    =& \frac{2\delta-2\delta(1-\delta)}{1-\delta} \\
    =& \ln\text{RHS}.
\end{align*}
Then we complete the proof of Lemma \ref{tight-chernoff}. \qed
\end{proof}

{\FTLmistake*}
	\begin{proof}
	We just have to prove that for large enough $K$, $\prob\{y_T = 1|\mathcal I_1\} \le \frac{1}{4}$ if $T \le \frac{c\ln K}{\varepsilon^2}$ for a small enough absolute constant $c$, and the lemma will be proved by the symmetrization of the Follow-the-Leader algorithm. 
	Let $X_{i,t}$ denotes the value of arm $i$ in time $t$, forall $i\in [K],t\in [T]$. Let event $A$ denote the event that FTL chooses arm $1$ after round $T$, and define $B$ to be the event such that
	\[\{\frac{1}{T}\sum_{t=1}^TX_{1,t} \le \frac{1+2\varepsilon}{2},\exists k>1,\frac{1}{T}\sum_{t=1}^TX_{k,t} > \frac{1+2\varepsilon}{2}\}\]
	Because $B\Rightarrow \lnot A$, so we have
	\[\prob\{A\} = 1-\prob\{\lnot A\} \le 1-\prob\{B\}.\]
	Then we provide a lower bound of the event $B$, which lead to an upper bound of the event $A$. From the tightness of chernoff bound (see Lemma \ref{tight-chernoff}), for large enougth $K$, we have
	\[\prob\{\frac{1}{T}\sum_{t=1}^TX_{k,t} > \frac{1+2\varepsilon}{2} = \frac{1}{2} + \varepsilon\} > e^{-c'T\varepsilon^2} \ge e^{-c'c\ln K},\]
	for $k > 1$ and an absolute constant $c'$,
	where the last inequality uses the assumption that $T \le \frac{c\ln K}{\varepsilon^2}$. 
	Then from the chernoff bound, we also have
	\[\prob\{\frac{1}{T}\sum_{t=1}^TX_{1,t} \le \frac{1+2\varepsilon}{2}\} \ge 1 - e^{-2T\varepsilon^2} \ge 1-e^{-c\ln K},\]
	where the last inequality uses the assumption that $T \ge \frac{c\ln K}{2\varepsilon^2}$. 
	Then from the independence of the arms, we have
	\[\prob\{B\} \ge (1-e^{-c\ln K})(1-(1-e^{-c'c\ln K})^{K-1}).\]
	Let $cc'\le \frac{1}{2}$ and let $K$ big enough, we have
	\[\prob\{A\} \le 1 - \prob\{B\} \le \frac{1}{4}.\]
	\qed
\end{proof}

{\FTLbest*}
	\begin{proof}
	To prove that there exists at least $\lceil K/3\rceil$ arms $j$ such that $\prob\{y_T = j|\mathcal I_j\} \le \frac{3}{4}$, we prove that
	\[\sum_{j=1}^K\prob\{y_T = j|\mathcal I_j\} \le \frac{K}{4},\]
	since if there are less than $\lceil K/3\rceil$ arms $j$ such that $\prob\{y_T = j|\mathcal I_j\} \le \frac{3}{4}$, we will have
	\[\sum_{j=1}^K\prob\{y_T = j|\mathcal I_j\} \ge (K - \lceil K/3\rceil) \times \frac{3}{4} > \frac{K}{4},\]
	for $K$ large enough. From the previous lemma, we just have to prove that Follow-the-Leader maximizes $\prob\{y_T = j|\mathcal I_j\}$. Let $\Omega$ denote the set of 
	all matrices $M\in \{0,1\}^{K\times T}$. Then for any fixed algorithm, we can view it as a function $\mathcal A:\Omega \to \mathbb R^K, \mathcal A(\omega) = (p_1(\omega),\dots,p_K(\omega))$, 
	where $p_k(\omega)$ denotes the probability that the algorithm $\mathcal A$ chooses the arm $k$ given a sample $\omega\in \{0,1\}^{K\times T}$. 
	Then we have
	\begin{align*}
	\sum_{j=1}^K\prob\{y_T = j|\mathcal I_j\} 
	=& \sum_{j=1}^K\sum_{\omega\in\Omega}\prob\{y_T = j|\omega, \mathcal I_j\}\prob\{\omega|\mathcal I_j\}\\
	=& \sum_{j=1}^K\sum_{\omega\in\Omega}\prob\{y_T = j|\omega\}\prob\{\omega|\mathcal I_j\}\\
	=&\sum_{j=1}^K\sum_{\omega\in\Omega}p_j(\omega)\prob\{\omega|\mathcal I_j\}\\
	=&\sum_{\omega\in\Omega}\sum_{j=1}^Kp_j(\omega)\prob\{\omega|\mathcal I_j\}.
	\end{align*}
	Note that $\prob\{\omega|\mathcal I_j\} = \left(\frac{1}{2}\right)^{(K-1)T} (\frac{1}{2}+\varepsilon)^{|\omega_{j\cdot}|}  (\frac{1}{2}-\varepsilon)^{T-|\omega_{j\cdot}|}  $,
	where $|\omega_{j\cdot}|$ is the number of $1$'s in the $j$-th row of sample matrix $\omega$.
	Follow-the-Leader algorithm would pick $j$ with the largest $|\omega_{j\cdot}|$, which also makes $\prob\{\omega|\mathcal I_j\}$ the largest.
	Therefore, Follow-the-Leader algorithm maximizes $\sum_{j=1}^Kp_j(\omega)\prob\{\omega|\mathcal I_j\}$ for any $\omega \in \Omega$, then we complete the proof. \qed
\end{proof}

{\THMlowerbound*}
\begin{proof}
	We first fix $\varepsilon$, and we will adjust this variable later. We also assume that $T \le \frac{c\ln K}{\varepsilon^2}$, where $c$ is the absolute constant in the previous lemma. Then for any bandit algorithm, let $I_t$ denote the arm chosen in round $t$. Then, denote
	\[S_t = \left\{\text{arms}\ j:\prob\{I_t = j|\mathcal I_j\}\le \frac{3}{4}\right\}.\]
	Then from the previous lemma, we know that $|S_t| \ge \frac{K}{3}$ when $\frac{c\ln K}{2\varepsilon^2} \le t \le \frac{c\ln K}{\varepsilon^2}$. 
	We consider a uniform distribution over problem instances $\mathcal I_j$. We want to prove that the expected regret lower bound on the instances is $\Omega(\sqrt{T\log K})$, and if we have the previous argument, there must exist an instance such that the regret is $\Omega(\sqrt{T\log K})$. Next, we lower bound the expected regret of the instances in a fixed round $t$. We first bound $\E[\mu_{I_t}|\mathcal I_j,j\in S_t]$, and we have
	\begin{align*}
	\E[\mu_{I_t}|\mathcal I_j,j\in S_t]
	\le& \prob\{I_t = j|\mathcal I_j,j\in S_t\}\E[\mu_{I_t}|I_t = j,\mathcal I_j,j\in S_t] \\
	&\quad + \prob\{I_t \neq j|\mathcal I_j,j\in S_t\}\E[\mu_{I_t}|I_t \neq j,\mathcal I_j,j\in S_t]\\
	=& \frac{1+\varepsilon}{2}\prob\{I_t = j|\mathcal I_j,j\in S_t\} + \frac{1}{2}\prob\{I_t \neq j|\mathcal I_j,j\in S_t\} \\
	=& \frac{1}{2} + \frac{\varepsilon}{2}\prob\{I_t = j|\mathcal I_j,j\in S_t\} \\
	\le& \frac{1}{2} + \frac{\varepsilon}{2}\cdot \frac{3}{4} \\
	=& \mu^* - \frac{\varepsilon}{8}
	\end{align*}
	With the previous result, we can bound the expected pseudo-regret in a fixed round $t$. We have
	\begin{align*}
	\E_{j\sim[K]}\E[\mu_{I_t}|\mathcal I_j] =& \sum_{j=1}^K\frac{1}{K}\E[\mu_{I_t}|\mathcal I_j] \\
	=& \frac{1}{K}\sum_{j=1}^K(\prob\{j\in S_t\}\E[\mu_{I_t}|\mathcal I_j,j\in S_t] + \prob\{j\notin S_t\}\E[\mu_{I_t}|\mathcal I_j,j\notin S_t]) \\
	\le& \frac{1}{K}\sum_{j=1}^K(\prob\{j\in S_t\}(\mu^* - \frac{\varepsilon}{8}) + \prob\{j\notin S_t\}\mu^*)\\
	=& \mu^* - \frac{1}{K}\sum_{j=1}^K\I\{j\in S_t\}\frac{\varepsilon}{8}\\
	\le& \mu^* - \frac{\varepsilon}{24}.
	\end{align*}
	Then choose $\varepsilon = \sqrt{c\ln K / T}$, and sum over $t$ which satisfies $\frac{c\ln K}{2\varepsilon^2} \le t \le \frac{c\ln K}{\varepsilon^2}$, we complete the proof. 
	\qed
\end{proof}

}

\end{document}